\documentclass{article}

\usepackage[preprint]{nips_2018}

\usepackage[utf8]{inputenc} 
\usepackage[T1]{fontenc}    
\usepackage{hyperref}       
\usepackage{url}            
\usepackage{booktabs}       
\usepackage{amsfonts}       
\usepackage{nicefrac}       
\usepackage{microtype}      
\usepackage[pdftex]{graphicx}
\usepackage{amssymb,amsmath,amsthm}
\usepackage{algorithm,algorithmicx,algpseudocode}

\newtheorem{theorem}{Theorem}
\newtheorem{corollary}{Corollary}
\newtheorem{assumption}{Assumption}

\title{Distributed Cartesian Power Graph Segmentation\\ for Graphon Estimation}

\author{
  Shitong Wei \\
  Department of Statistics\\
  UC Davis\\
  Davis, CA 95618 \\
  \texttt{swei@ucdavis.edu} \\
  \And
  Oscar Hernan Madrid-Padilla \\
  Department of Statistics\\
  UC Berkeley\\
  Berkeley, CA 94720 \\
  \texttt{omadrid@berkeley.edu} \\
  \And
  James Sharpnack \\
  Department of Statistics\\
  UC Davis\\
  Davis, CA 95618 \\
  \texttt{jsharpna@ucdavis.edu} \\
}

\begin{document}

\maketitle

\begin{abstract}
We study an extention of total variation denoising over images to over Cartesian power graphs and its applications to estimating non-parametric network models.
The power graph fused lasso (PGFL) segments a matrix by exploiting a known graphical structure, $G$, over the rows and columns.
Our main results shows that for any connected graph, under subGaussian noise, the PGFL achieves the same mean-square error rate as 2D total variation denoising for signals of bounded variation.
We study the use of the PGFL for denoising an observed network $H$, where we learn the graph $G$ as the $K$-nearest neighborhood graph of an estimated metric over the vertices.
We provide theoretical and empirical results for estimating graphons, a non-parametric exchangeable network model, and compare to the state of the art graphon estimation methods.
\end{abstract}

\section{Introduction}

Total variation (TV) denoising, also known as the fused lasso, is a classical method for image denoising \cite{chambolle1997image} that groups pixels that are adjacent to one another and have similar pixel values, a process known as segmentation.
For a network, $H$, the analogous task is to segment all possible vertex pairs by segmenting the adjacency matrix of the network. 
While it does not make sense to segment based on the ordering of the vertices, as in TV denoising, if we have some other graph structure, $G$, over the vertices of $H$, then there is some hope of segmenting vertex pairs based on proximity in $G$.
This paper studies the natural generalization of TV denoising over images to networks, using a known or learned graph $G$ to provide the structure. 
(Throughout, we will call the response graph, $H$, a network, and the predictor graph, $G$, a graph.)
To this end, we will introduce the power graph fused lasso (PGFL), and discuss learning the graph, $G$, for graphon models, a non-parametric network model \cite{bickel2009nonparametric}.

\subsection{Methodological overview: denoising a network with KNN-PGFL}

{\bf Power graph fused lasso} is one approach to denoising a response matrix $A$ with a known graph $G = (V,E)$, where the $i$th row and column of $A$ correspond to the $i$th vertex in $G$ ($n:= |V|$).
For example, the underlying graph structure may be based on individuals' spatial proximity while the response matrix, $A_{i,j}$ may be binary indicating whether or not the individuals are friends on a social network.
Our approach will partition the set of all possible dyads, $V^{\times 2} := \{(i,j): i, j \in V\}$, based on the friendship status of these pairs of individuals.
Throughout, we will call elements of $V$ vertices, and elements of $V^{\times 2}$ dyads.
We will approach this problem by constructing a graph over the set of all dyads, $V^{\times 2}$, called the Cartesian power graph of degree 2, or the C2-power graph for short.
Thus $A_{i,j}$ is a label for the dyad $(i,j)$ which is a node in the C2-power graph, and we will study the fused lasso over the C2-power graph to denoise $A$.

Specifically, define the {\em Cartesian power graph of degree 2} (C2-power), $G^{\square 2} = (V^{\times 2}, E^{\square 2})$, where two dyads $(i_0,j_0), (i_1,j_1)$ have an edge connecting them if there is an edge between vertices in one coordinate and the vertices are equal in the other.
Specifically, the C2-power graph edge set is
\begin{equation*}
E^{\square 2} := \left\{ ((i_0,j_0), (i_1,j_1)) \in V^{\times 2}: \left( i_0 = i_1 \textrm{ and } (j_0, j_1) \in E \right) \textrm{ or } \left( (i_0, i_1) \in E \textrm{ and } j_0 = j_1 \right) \right\} 
\end{equation*}
This is consistent with the well-known notion of a graph Cartesian product and the C2-power graph is the Cartesian product of $G$ with itself, (commonly denoted $G^{\square 2} = G \square G$).
Throughout, we will let $A_{i,j} \in \mathbb R$ be our response matrix (supervising variable), which may be binary or continuous; for graphon estimation it is an adjacency matrix.
Throughout, we will consider a directed supervising variable in which $(i,j)$ is an ordered pair and $A$ may be asymmetric, but much of the results and methods can be extended to the case in which $(i,j)$ denotes unordered pairs.


\begin{figure}
\hspace{-20mm}
\begin{center}
  \includegraphics[width=.8\textwidth]{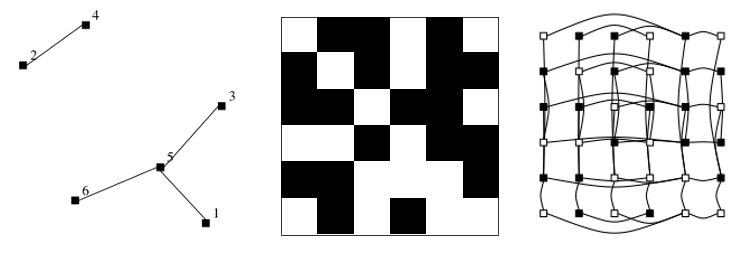}
\end{center}
  \caption{KNN-PGFL method: the 2-NN graph $G_2$ (left) is learned from the adjacency matrix $A$ (middle) of the network $H$, then PGFL is applied to $G_2^{\square 2}$ with the $A_{ij}$ dyadic labels (right).}
  \label{fig:PGFL}
\end{figure}

A natural approach to segmenting a column $A_i$ based on the graph $G$ is to use the fused lasso, also known as TV denoising.
The fused lasso seeks to minimize the data driven loss while also maintaining a small total variation, and over a graph it can be written,
\begin{equation}
\label{eq:FL}
\min_\beta \| y - \beta \|_2^2 + \lambda \| \nabla \beta \|_1,
\end{equation}
where $y = A_i$, $\beta \in \mathbb R^n$, $\lambda > 0$, and $\nabla$ is the $m \times n$ matrix such that $\| \nabla \beta \|_1 = \sum_{(i,j) \in E} |\beta_i - \beta_j|$ ($\nabla$ is known as the edge incidence matrix).
The effect of the TV norm $\| \nabla \beta \|_1$ is to attract nearby vertex values to one another.
Due to the nature of $\ell_1$-type norms, the solutions will tend to be piecewise constant over the graph, where for $\lambda$ large enough, there will be just a few clusters within which the values of $\beta$ are identical.
Because we would like to simultaneously denoise the rows and columns of $A$, we cannot use the fused lasso individually over each row. 
Instead, we will denoise the entire matrix $A$ by applying the fused lasso over the C2-power graph.
{\em Power graph fused lasso} (PGFL) is the solution to the following program,
\begin{equation}
\label{eq:PGFL}
\min_{P\in \mathbb R^{n \times n}} \|A - P\|_{F}^2 + \lambda \left( \| \nabla P \|_1 + \| \nabla P^\top \|_1\right),
\end{equation}
where we abuse notation to allow $\| B \|_1 = \sum_{i,j} |B_{i,j}|$ to be the matrix 1-norm and $\|.\|_{F}$ is the Frobenius 2-norm.
To see that the RHS of \eqref{eq:PGFL} is the TV penalty over the C2-power graph, notice that 
\begin{equation*}
\| \nabla P \|_1 + \| \nabla P^\top \|_1 = \sum_{k=1}^n \| \nabla P_k \|_1 + \| \nabla (P^\top)_k \|_1 = \sum_{(i,j) \in E, k \in V} \left( |P_{k,i} - P_{k,j}| + |P_{i,k} - P_{j,k}|\right),
\end{equation*}
where $P_k$ is the $k$th column of $P$.  You can see a vignette of the C2-power graph in Figure \ref{fig:PGFL}.


{\bf Graphons} are network models that provide a non-parametric representation of exchangeable graph models (see \cite{diaconis2007graph} for a thorough introduction).
Let $f: [0,1]^2 \to [0,1]$ be a graphon and let $\{\xi_i\}_{i=1}^n$ be iid draws from a uniform$[0,1]$ distribution.
The graphon model assumes that conditional on the latent variables, $\{\xi_i\}_{i=1}^n$, the adjacency matrix is a Bernoulli ensemble, $A_{i,j} \sim {\rm Bernoulli}(P_{i,j})$ where $P_{i,j} = f(\xi_i, \xi_j)$.
This forms the network, $H$, and the model is exchangeable because the probability distribution is invariant under permutation of the vertices.
This highlights the fundamental challenge that is inherent in estimating graphon models: in order to estimate $P$, we must account for the nuisance parameters $\xi_i$.

We will approach the graphon estimation problem by constructing the graph $G$ from the network $H$ and applying the PGFL to get an estimated probability matrix, $\hat P$.
Recently, \cite{zhang2015estimating} proposed to estimate a metric between vertices $i,j$ based on the adjacency matrix, $A_{i,j}$, and then apply a neighborhood smoother on each column, $A_i$, separately.
The metric that they construct is 
\begin{equation}
\label{eq:dinfty}
\hat d_\infty^2 (i,j) := \max_{k \ne i,j} |(A_i - A_j)^\top A_k| / n,
\end{equation} 
which they show will approximately bound the desired but unknown metric, \smash{$d_2(i,j) := \| f(\xi_i,.) - f(\xi_j,.) \|_{L_2}$} under Lipschitz assumptions. 
We will introduce a similar metric, $\hat d_1$, and use it to form the K-nearest neighbor graph, $G_K$, between the vertices of $H$.
The learned graph $G_K$ is our proxy for the unknown latent parameters $\xi_i$, and the KNN-PGFL is the application of the PGFL, \eqref{eq:PGFL}, to $A$ using the C2-power graph, $G_K^{\square 2}$. 



\subsection{Related Work and Contributions}

Graph signal processing refers to methods that denoise, localize, detect, and predict signals over graphs.
For example, each vertex corresponds to a low-powered sensor, and we would like to denoise sensor measurements, and we use the graph structure is based on communication between the sensors or spatial proximity.
The driving assumption is that there is some underlying signal that in some way `respects' the graph topology, and specifies the distribution of the observations.
Many of the tools in signal processing and supervised learning can be extended to the graph case, such as Fourier analysis \cite{sandryhaila2013discrete,hu2015multiresolution}, wavelets \cite{crovella2003graph, sharpnack2013detecting, irion2014generalized}, graph kernels \cite{smola2003kernels}, and convolutional networks \cite{kipf2016semi, henaff2015deep}. 
Graph structure has previously been used in matrix completion and network denoising problems (see for example, \cite{cai2011graph, gu2010collaborative, liu2015bipartite, brunner2012pairwise}), but these methods require some predetermined graph structure, such as knowledge graphs, so are not well suited to estimating graphons, and they do not perform segmentation, which is the focus of this work.

There is an extensive body of literature on solving the fused lasso, \eqref{eq:FL}.
Algorithms for solving the fused lasso can be divided into two categories: solvers for a fixed $\lambda$, and path algorithms that find the solution for every $\lambda$ within a range.
The fused lasso for a fixed $\lambda$ has a quadratic program dual form, and some popular algorithms for this are the projected Newton algorithm of \cite{bertsekas1982projected, barbero2011fast}, first-order primal-dual algorithm \cite{chambolle2011first}, and split-Bregman iteration \cite{goldstein2009split}.
Some path algorithms include the generalized lasso path algorithm of \cite{arnold2016efficient}, and a max-flow version for the fused lasso in \cite{hoefling2010path}.
If applied directly to C2-power graphs, these methods would have computational and memory complexity that scale with the number of dyads, $n \choose 2$.

{\bf Contribution 1.} We provide a distributed implementation of the power graph fused lasso (PGFL), \eqref{eq:PGFL}, based on a novel formulation using the alternating direction method of multipliers.

Recent theoretical studies have examined statistical rate guarantees for the fused lasso over graphs, \eqref{eq:FL}.
In these works, it is assumed that the true signal is either of bounded variation $\| \nabla \beta \|_1 \le C$ or is piecewise constant \smash{$\| \nabla \beta \|_0 = |\{(i,j) \in E : \beta_i \ne \beta_j \}|\le C$}, for some constant $C > 0$ (although all discrete signals are technically piecewise constant, we refer to signals with bounded number of changepoints as piecewise constant signals).
\cite{sharpnack2012sparsistency} provided conditions under which one could exactly localize changepoints, edges across which the underlying signal ($\beta$) changes, under the piecewise constant assumption.
These conditions were too strict to be realistic and it was discovered that for many graphs, the mean-square error (MSE) of the fused lasso could diminish even though the changepoints are not precisely recovered.
For 1D chain graphs of length $n$ the mean square error (MSE) was shown to be diminishing like $n^{-2/3}$ for functions of bounded variation \cite{wang2016trend}, and $n^{-1} {\rm plog~} n$ for piecewise constant functions under mild conditions, \cite{lin2017sharp, guntuboyina2017spatial} (throughout plog will refer a poly-logarithmic term).
\cite{padilla2016dfs} demonstrated that the MSE scales like $n^{-2/3}$ for all connected graphs, and not just the 1D chain graph, for functions of bounded variation.
For the 2D grid graph of size $n \times n$, \cite{hutter2016optimal} demonstrated that the MSE diminishes like $n^{-2} {\rm plog~} n$ for both signals of bounded variation and piecewise constant signals. 
The 2D grid graph is the C2-power graph of the 1D chain graph, and so it is reasonable to hope that the 2D grid graph actually has the slowest convergence of any C2-power graph.

{\bf Contribution 2.} We prove that for any connected C2-power graph, $G^{\square 2}$, the mean square error of the PGFL diminishes like $n^{-2} {\rm plog~}n$ when the signal $\mathbb E A$ is of bounded variation and $A$ is subGaussian.

We next turn our attention to graphon estimation using the PGFL on a learned graph.
The statistical limits of graphon estimation have been well characterized for smooth graphons, and it was found that computationally intractable profile likelihood maximization is minimax optimal for H\"older graphons, \cite{wolfe2013nonparametric, gao2015rate}.
One tractable approach to graphon estimation is to order the vertices according to some graph statistics, such as the degrees of $H$, and then treat the resulting re-ordered matrix $(A_{\pi(i),\pi(j)})_{i,j}$ as an image and applying image segmentation tools ($\pi$ is the permutation associated with this sorting).
This methodology is called sorting and smoothing (SAS), and in \cite{chan2014consistent} they use TV denoising to perform the image segmentation.
The implicit assumption is that the degree is a decent proxy for the latent variable, $\xi_i$, which does not hold for most graphons.

Another related approach to segment the dyads is to group the vertices via a community detection method.
The stochastic block model is a special instance of the graphon model which assumes that there are latent communities for the vertices and the probability of attachment between two vertices is a function only of the communities to which the vertices belong. 
This can be thought of as segmenting the dyads by taking the Cartesian product of the vertex communities, but this type of segmentation is restrictive because of this specialized structure.
Heuristic or greedy methods for fitting the SBM for graphon estimation have been proposed in \cite{airoldi2013stochastic, cai2014iterative}, but little is known about the statistical performance and whether these can achieve minimax performance.
In another approach, \cite{chatterjee2015matrix} proposed a spectral method that thresholds singular values and provided some MSE consistency guarantees.
Currently, the best rate guarantee for a computationally tractable estimator of Lipschitz graphons is achieved by the aforementioned neighborhood smoothing method of \cite{zhang2015estimating}, and the MSE scales like $n^{-1/2} {\rm plog~}n$, which is significantly worse than the minimax rate of $n^{-1} {\rm plog~}n$.


{\bf Contribution 3.} We propose the K-nearest neighbors power graph fused lasso (KNN-PGFL) for graphon estimation, compare its empirical performance to other graphon estimators, and provide theoretical guarantees under a bounded variation assumption on the graphon and additional conditions.

\section{Method}

\subsection{Distributed power graph fused lasso}


In this section, we provide a distributed method for solving PGFL, \eqref{eq:PGFL}, by iterating parallel row-wise and column-wise operations. 
Our method uses the alternating direction method of multipliers to separate the two terms of the TV penalty on the C2-power graph.
If we make the substitution $Q := P^\top$ then we can reformulate \eqref{eq:PGFL} as 
\begin{align*}
  \min_{P, Q \in \mathbb R^{n \times n}} \frac 12 \| A - P \|_F^2 + \frac 12 \| A^\top - Q \|_F^2 + \lambda \| \nabla P \|_1 + \lambda \| \nabla Q \|_1 ~ {\rm s.t.} ~ P &= Q^T.
\end{align*}
The augmented Lagrangian, with multiplier $U \in \mathbb R^{n \times n}$, for this primal problem is
\begin{align*}
\frac 12 \| A - P \|_F^2 + \frac 12 \| A^\top - Q \|_F^2 + \lambda \| \nabla P \|_1 + \lambda \| \nabla Q \|_1 + \langle U,P-Q^T \rangle+\frac{\eta}{2}||P-Q^T||_F^2,
\end{align*}
where $\langle . , . \rangle$ is the trace inner product.
When $U, Q$ is fixed, the minimization wrt $P$ takes the form of the separable minimization,  
\begin{equation*}
\min_{P \in \mathbb R^{n \times n}} \frac{1 + \eta}{2}\| P - \tilde A \|_F^2 + \lambda \| \nabla P \|_1 = \sum_{i=1}^n \min_{P_i \in \mathbb R^n} \frac{1 + \eta}{2} \| P_i - \tilde A_i \|_2^2 + \lambda \|\nabla P_i \|_1,
\end{equation*}
for some matrix $\tilde A$ (and vice versa for $U,P$ fixed).
The inner minimization of the RHS is the fused lasso on the graph $G$ (the prox operator for the graph total variation), which we can take to be an algorithmic primitive.
Let \smash{$\textnormal{prox}_{\lambda,G}(y) := \arg\min_\beta \| y - \beta \|_2^2 + \lambda \| \nabla \beta \|_1$} be the proximal operator, then we can summarize the resulting ADMM algorithm in Algorithm \ref{alg:PGFL}.

\begin{algorithm}
\label{alg:PGFL}
\caption{Distributed power graph fused lasso (PGFL)}
\hspace*{\algorithmicindent} \textbf{Input:} Graph $G$, response matrix $A$, tuning parameter $\lambda > 0$\\
\hspace*{\algorithmicindent} \textbf{Output:} Denoised matrix $\hat P$
\begin{algorithmic}[1]
  \State Initialize $P = \mathbf{0}$, $Q = A^T$, $U = \mathbf{0}$,
  \While{stopping criteria not met}
  \State parallel for $P_i \gets \textnormal{prox}_{\frac{2\lambda}{1+\eta},G}(\frac{A_i-U_i+\eta (Q^T)_i}{1+\eta})$,
  \State parallel for $Q_i \gets \textnormal{prox}_{\frac{2\lambda}{1+\eta},G}(\frac{(A^T)_i+(U^T)_i+\eta (P^T)_i}{1+\eta})$,
  \State $U \leftarrow U +  \eta(P-Q^T)$.
  \EndWhile
  \State $\hat P \gets P$
\end{algorithmic}
\end{algorithm}

We use projected Newton iteration to compute the proximal operator (see the Appendix for the exact specification), which requires a Laplacian system solver.
Projected Newton maintains an active set of edges $E$, such that if $(j,k)$ is in the active set then $P_{i,j} = P_{i,k}$ (when prox is applied to the $i$th row/column), and similarly for $Q_i$ with a different active set.  
Hence, the denoised matrix, $\hat P$, will have regions of constant value that are connected by elements of the active sets, and in this way it will segment the matrix $A$. 
This methodology works for any graph $G$ and response matrix $A$, in the next section we outline the application of the PGFL for graphon estimation.

\subsection{Fused graphon estimation}

Suppose that the response matrix $A$ is the adjacency matrix for an observed network $H$, and we are tasked with estimating the underlying probability matrix $P_0 = \mathbb E A$.
A natural approach is to begin with a metric that is extracted from $H$, then forming the K-nearest neighbor (KNN) graph for this metric.
The idea is that if the underlying graphon is of sufficiently controlled variation with respect to the metric, then the variation between KNNs will be likewise controlled.

Constructing a meaningful metric over the vertices of the graphon is challenging because there are only a few statistics of the graphon that can be reliably estimated.
Particularly, \cite{zhang2015estimating} observed that the inner product $\int f(\xi_i,.) \cdot f(\xi_j,.)$ has the unbiased estimator $A_i^\top A_j / n$ when $i \ne j$, but $A_i^\top A_i/n$  (the degree of the vertex $i$ divided by $n$) has expectation $\int f(\xi_i,.)$.
So, estimating the $L_2$ norm between the graphon cross-sections (and most other common norms), is exceedingly difficult.
\cite{zhang2015estimating} approached this problem by approximating the $L_2$ metric with $\hat d_\infty$, \eqref{eq:dinfty}.
We propose the use of a similar metric, which we empirically observe to be a more stable variant, 
\begin{equation*}
\hat d_1^2(i,j) := \frac{1}{n(n-2)} \sum_{k \ne i,j} |(A_i - A_j)^\top A_k|.
\end{equation*}
We then generate the KNN graph, $G_K$, which is defined to be symmetric and undirected, by connecting edges if either of the incident vertices is a K-nearest neighbor of the other.
By applying the PGFL to $G_K$ and $A$, then we obtain a $\hat P$ which will be piecewise constant.
Finally, we obtain a partition of the dyads, $V^{\times 2}$, which are those regions of the C2-power graph, $G_K^{\square 2}$, over which $\hat P$ is constant.

\begin{algorithm}
\label{alg:GS}
\caption{K-nearest neighbors power graph fused lasso (KNN-PGFL)}
\hspace*{\algorithmicindent} \textbf{Input:} network $H$ with adjacency matrix $A$, tuning parameter $\lambda > 0$\\
\hspace*{\algorithmicindent} \textbf{Output:} partition of $V^{\times 2}$, $\mathcal S$, and estimated probabilities $\hat P$
\begin{algorithmic}[1] 
  \State Calculate the $\hat{d}_1$ distance matrix $\hat{D}_1 = (\hat{d}_1(i,j))_{i,j}$;
  \State Generate the undirected KNN graph $G_K$: $(i,j) \in E$ if $i$ is a KNN of $j$ or vice versa;
  \State Calculate $\hat P$ with Distributed PGFL on $G_K, A, \lambda$.
  \State Augment $G_K^{\square 2}$ by removing non-active edges $((i_0,j_0),(i_1,j_1)) \in E^{\square 2}$ such that $\hat P_{i_0,j_0} \ne \hat P_{i_1,j_1}$.
  \State Return the connected components of the augmented C2-power graph, $\mathcal S$, and $\hat P$.
\end{algorithmic}
\end{algorithm}

\section{Theory}

We will begin our theoretical analysis with a mean-square error guarantee for the PGFL on any graph $G$.
This will give us corollaries for graphon estimation according to Algorithm \ref{alg:GS}.

\subsection{Guarantees for general power graphs}
Recall that \cite{hutter2016optimal} demonstrated that the MSE of total variation denoising of a 2D image scales like $n^{-2} {\rm plog~} n$ under a bounded variation assumption.
2D total variation denoising is the PGFL when $G$ is the 1D chain graph, and for our main result, we find that this rate guarantee holds for any connected graph, $G$.
To prove this result we use a proof technique pioneered in \cite{padilla2016dfs}, where the depth-first search algorithm is used to reorder the vertices in a way that approximately preserves total variation.
We modify this technique to work for C2-power graphs, and arrive at our desired conclusion.

\begin{theorem}[PGFL for general $G$]
	\label{thm:general}
	Suppose that $A$ has expectation $P_0$, and that $G$ has $q$ connected components.
	Let $\hat P$ be the solution to \eqref{eq:PGFL}, let $R = A - P_0$, and assume each entry, $R_{i,j}$, is an independent and subGaussian($\sigma^2)$.
	Then for some choice of $\lambda \asymp \log n \sqrt{\log (nq)}$ the MSE decays,
	\begin{equation*}
	\frac{1}{n^2}\| \hat P - P_0 \|_F^2 = O_{\mathbb{P}} \left(  \frac{q^2 \log q}{n^2}  + \frac{q^2   \| \hat P - P_0\|_{\infty}  \log n}{n}   +   \frac{\log n}{n^2}(\| \nabla P_0 \|_1 + \| \nabla P_0^\top \|_1) \right).
	\end{equation*}
\end{theorem}

We note that if the graph $G$  is connected, then $q=1$, and from the proof we see that  the term  involving $\|\hat P - P_0\|_{\infty}$ disappears  from the upper bound.


\subsection{Graphon estimation}
Algorithm \ref{alg:GS}, is predicated on the notion that if you consider the K-nearest neighbors of vertex $i$ in $\hat d_1$, then these will have similar graphon cross-sections, namely, $f(\xi_i,.) \approx f(\xi_j,.)$ for neighbor $j$.
When $f$ is sufficiently smooth, then $\xi_i \approx \xi_j$ will imply that the corresponding graphon cross-sections are similar.
In this work, the notion of smoothness that we will assume is that the cross-sections are of bounded variation.

\begin{assumption}
\label{as:bv}
There exists a constant $B>0$, such that for  any  $v \in [0,1]$   and  $0  \, \leq \,u_1  \,\leq  \, u_2\,\leq \,\ldots\,\leq u_s \,\leq \, 1$  for  $s \in  \mathbb{N}$  we have that the graphon $f$ satisfies,
\begin{equation*}
  \displaystyle \sum_{l=1}^{s-1} \,  \left\vert  f(u_s,v)  \,-\,  f(u_{s+1},v)  \right\vert  \,\,\leq\,\,B, \quad   \displaystyle \sum_{l=1}^{s-1} \,  \left\vert  f(v,u_s)  \,-\,  f(v,u_{s+1})  \right\vert  \,\,\leq\,\,B.
\end{equation*}
\end{assumption}

Our proposed Algorithm \ref{alg:GS} is a departure from the neighborhood smoother of \cite{zhang2015estimating} in two ways: we use the metric, $\hat d_1$, instead of $\hat d_\infty$, and the PGFL provides a segmentation of the entire adjacency matrix, $A$ (as opposed to smoothing in a row-wise fashion).
We find in this section that the performance of the KNN-PGFL is very dependent on the quality of our underlying metric $\hat d_1$, which we find to be more stable than $\hat d_\infty$.
This is consistent with the theoretical results in \cite{zhang2015estimating}.
Roughly, speaking, the statistical rate bottleneck in their analysis lies with variability of their metric $\hat d_\infty$.
One can imagine that because $\hat d_\infty^2(i,j)$ is based on the average of $n$ independent random variables, it will have a standard error of around $n^{-1/2}$.
Notably this error is additive, meaning that even when $\xi_i \approx \xi_j$, we may have that $\hat d_\infty^2(i,j)$ is on the order of $n^{-1/2}$.
This measurement error means that the resolution for estimating a smooth graphon using $\hat d_\infty$ will be at the scale of $n^{-1/2}$, which is significantly different from the optimal resolution of $n^{-1}$---we can smooth at bandwidths that are on this order and obtain optimal graphon estimators.
This additive error term, $\Delta_n$, is made precise in the following assumptions which apply to any choice of metric.

\begin{assumption}
\label{as:hat_lower}
The distance $\hat d$ is lower Lipschitz wrt $\xi$ with constant $L_1>0$, and additive error $\Delta_n >0$, if for $(i,j) \in V^{\times 2}$,
\begin{equation*}
L_1 |\xi_i - \xi_j| - \Delta_n \le \hat d^2(i,j).
\end{equation*}
\end{assumption}

\begin{assumption}
\label{as:hat_upper}
The distance $\hat d$ is piecewise Lipschitz wrt $\xi$ with constant $L_2>0$, and additive error $\Delta_n >0$, if the following holds.
There exists a constant $L_2 > 0$  and  a partition  $ \mathcal{A}\,:=\,\{ a_1,\ldots,a_{m -1}  \}$  and sets  $A_1  \,=\,[a_0,a_1)$ with $a_0 \,=\,0$, and $A_l \,\,=\,\, [a_{l-1},a_{l})  $ for $l \in \{2,\ldots,m-1\}$, and   $A_{m} \,=\,[a_{m-1},a_{m}]$   with $a_{m} \,=\,1$, such that for $(i,j) \in V^{\times 2}$,
\begin{equation*}
\xi_i,\xi_j \in A_l,  \,\,\,l \in \{1,\ldots,m-1\},\,\,\,\,\,\,\,\, \text{implies}\,\,\,\,\,\,\,\,   \hat{d}^2(i,j) \,\leq\,L_2\, \vert \xi_i\,-\,\xi_j \vert \,+\,\Delta_n.
\end{equation*}
\end{assumption}

Assumption \ref{as:hat_lower} is a statement that the cross-sections do not repeat themselves in the sense that if $\xi_j$ is far from $\xi_i$ then the the corresponding cross-sections are sufficiently different in the metric.
Assumption \ref{as:hat_upper} will hold for $\hat d_1$ if we assume that the graphon is piecewise Lipschitz where $\Delta_n \asymp n^{-1/2} {\rm plog~}n$ (see the Appendix of \cite{zhang2015estimating} for a similar derivation).
If we have a metric and graphon that satisfies these assumptions, then we can obtain an MSE rate bound that is dependent on $\Delta_n$.

\begin{corollary}
	\label{cor:graphon}
	Suppose that $A \in \mathbb R^{n \times n}$ is drawn from a graphon model with graphon $f$ that satisfies Assumption \ref{as:bv}, and let $P_{0,i,j} := f(\xi_i,\xi_j)$ be the conditional edge probability. 
	Let $\hat P$ be the output of Algorithm \ref{alg:GS} applied to $A$ with a metric $\hat d$ that satisfies Assumptions \ref{as:hat_lower}, \ref{as:hat_upper}, and for $K/ \log n \rightarrow \infty$ and $K/n \rightarrow 0$.
	Suppose that the KNN graph, $G_K$, has $q$ connected components, then there is a choice of $\lambda$ such that 
	\begin{equation*}
	\frac{1}{n^2} \| \hat P - P_0 \|_F^2 = O_P \left( \frac{q^2 \log n}{n}    +   \frac{(K^2 + n K \Delta_n)\cdot\log n }{n}\right).
	\end{equation*}
\end{corollary}

The MSE bound of Theorem \ref{cor:graphon} is dependent on the additive error of the metric, $\Delta_n$.
In the event that we find a metric with additive error $\Delta_n \asymp n^{-1}$, and the KNN graph is connected, then the KNN-PGFL can achieve near minimax rates (unfortunately, all known metrics have an error $\Delta_n \asymp n^{-1/2}$).
Instead of making our assumptions about $\hat d_1$, we make these assumptions about the population level version of the metric,
\begin{equation*}
d_1^2(i,j) := \int_0^1 \left| \int_0^1 \left( f(\xi_i,v) - f(\xi_j,v) \right) f(u,v) dv \right| du.
\end{equation*}
We now consider these assumptions placed on $d_1$ instead of $\hat d_1$.

\begin{corollary}
\label{cor:graphon_dhat}
Suppose that Assumptions \ref{as:hat_lower}, \ref{as:hat_upper} hold for $d_1$ with $\Delta_n = O_P(\sqrt{(\log n)/n})$ then they also hold for $\hat d_1$ with $\Delta_n \asymp \sqrt{(\log n)/n}$.
With these assumptions on $d_1$ and under the remaining conditions of Corollary \ref{cor:graphon} and assume $q = O_P(n^{1/4} \cdot {\rm plog~} n)$ we set $K = O({\rm plog~}n)$, the KNN-PGFL with metric $\hat d_1$ has MSE bound
\begin{equation*}
  \frac{1}{n^2} \| \hat P - P_0 \|_F^2 = O_P \left( \frac{{\rm plog~}n}{\sqrt n}\right).
\end{equation*}
\end{corollary}

This result is consistent (up to logarithmic terms) with what was found in \cite{zhang2015estimating}, although under somewhat different conditions.
It is outside of the scope of this work to comprehensively study the construction of better metric, and we believe that a significant departure from $\hat d_1$ and $\hat d_\infty$ is needed.

\section{Experiments}


To test the empirical performance of KNN-PGFL, we simulate from five graphon models and evaluate the mean-square error of some important graphon estimators.
In addition to our own, four other methods were used for comparison, neighborhood smoothing (NS), \cite{zhang2015estimating}, sorting and smoothing (SAS), \cite{chan2014consistent}, the stochastic block model (SBM), \cite{airoldi2013stochastic}, and USVT, \cite{chatterjee2015matrix}. 
For each graphon function and each repetition, a graph with 1000 nodes was generated. 
NS, SAS, USVT and KNN-PGFL were applied to the same graph. 
The penalty parameter $\lambda$ in KNN-PGFL was chosen as $\lambda = 0.5$ for all graphon functions, and 2-Nearest Neighborhoods was used (this gave a well connected KNN graph with only a few connected components). 
The stopping criterion for KNN-PGFL was $\|P-Q^T\|_F< 0.01\|Q\|_F$, and the resulting $Q$ was used as the estimated probability matrix. 
For the SAS method, the bandwidth parameter $h$ was chosen as $10$.  
For SBM, at least two observed graphs are needed, so to make the comparison fair, for each repetition, 4 graphs with 500 nodes were generated according to the graphon function. 
SBM was applied to the four observed graphs, and the tuning parameter ($\Delta$ in their paper) was chosen by cross-validation. 
The MSEs were averaged over 30 repetitions, multiplied by $10^4$, are shown in Table~\ref{tab:sim_MSE}.

\begin{table}
  \caption{Mean-square error comparisons}
  \label{tab:sim_MSE}
  \centering
  \begin{tabular}{llllll}
    \toprule
    Method & Graphon A & Graphon B & Graphon C & Graphon D & Graphon E \\
    \midrule
    KNN-PGFL &7.39  &{\bf 3.10} &17.54 & {\bf 34.91} &{\bf 61.08} \\
    Neigh. Smooth&13.68 &9.55 &17.16&45.18 &66.76 \\
    SAS &{\bf 6.29}  & 9.20 &23.68&97.90  & 190.38 \\
    SBM&37.65  &6.60&35.77  &  44.45&62.68 \\
    USVT& 7.05 & 9.61&{\bf 12.24} & 50.34 &71.94 \\
    \bottomrule
  \end{tabular}
\end{table}

The performance of each estimator is bound to be highly dependent on the structure of the graphon (see Figure \ref{fig:graphons} for the graphons and their estimates).
Graphon A has monotonic node degrees, and is of low rank; as a result SAS and USVT perform well in this case, but KNN-PGFL works similarly to these as well.
Graphon B is a graph with blocks, and also a piecewise constant function; KNN-PGFL performs best, followed by SBM which is designed for this situation.
Graphon C is a smooth graphon function with local structure, and the best result is obtained by USVT, followed by NS and KNN-PGFL, but none of the methods are especially well suited to this graphon.
Graphon D and E are both piecewise constant graphon functions. Due to the lack of monotonicity here, SAS fails to recover the probability matrix.  KNN-PGFL gives the best MSE results, followed by SBM. 
For all five graphons, KNN-PGFL performs well and does not catastrophically fail, and in all but one case, it significantly outperforms other segmentation methods, SAS and SBM.

\begin{figure}
\hspace{-20mm}
\begin{center}
  \includegraphics[width=.9\textwidth]{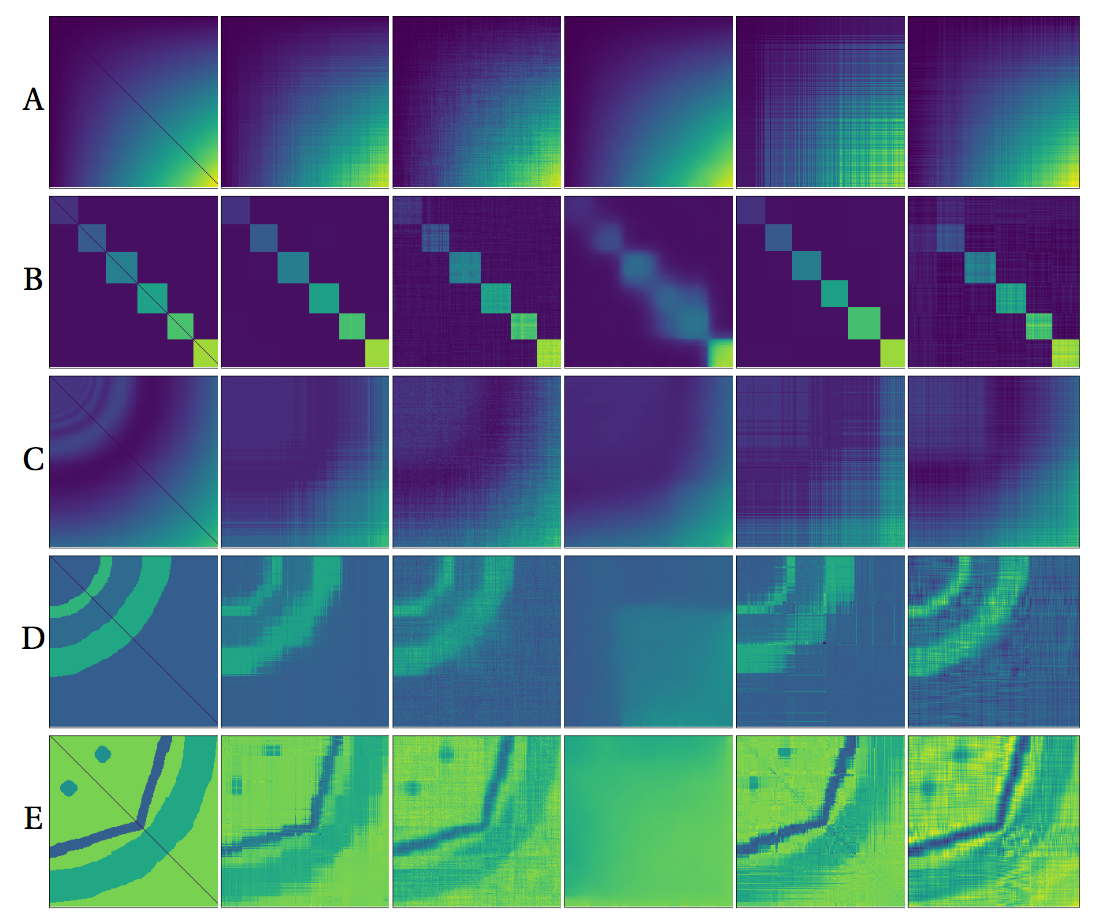}
\end{center}
  \caption{Graphons A, B, C, D and E are shown respectively. For each row, from the left to the right, the plots are the true probability matrix, the estimates using PGFL, NS, SAS, SBM, and USVT.}
  \label{fig:graphons}
\end{figure}

\section{Conclusion}

We proposed the power graph fused lasso for denoising a matrix with a known graph over the rows and columns. 
Our main theorem, \ref{thm:general}, demonstrates that it achieves the same mean-square error guarantee as 2D total variation denoising under a subGaussian error model.
We proposed its use for graphon estimation with the K-nearest neighbors graph, and studied its performance both theoretically and empirically.
We find that it is empirically competitive with existing methods and significantly outperforms the other graphon segmentation methods, SAS and SBM.
Theoretically and experimentally, we find that the performance of KNN-PGFL is limited by the quality of the distance metric $\hat d_1$, due to the additive error characterized in Assumptions \ref{as:hat_lower}, \ref{as:hat_upper} (a similar problem shared by neighborhood smoothing, \cite{zhang2015estimating}).
We hope that future work can discover better vertex metrics for graphon models that can be used in conjunction with the proposed methodology.

\subsection*{Acknowledgements}

JS is partially supported by NSF DMS-1712996.

\bibliographystyle{apalike}
\bibliography{PGFL}

\newpage
\section{Appendix for ``Distributed Cartesian Power Graph Segmentation for Graphon Estimation''}

\begin{proof}[Proof of Theorem \ref{thm:general}.]
	We will follow a standard derivation of MSE bound for penalized estimation with minor modifications.
	See, for example, \cite{wang2016trend} for many of these tools.
	Let  us also  denote  by  $G_{l} \,=\,(E_l,V_l)$,  $l \in [q]$  the connected components of $G$. 
	We also write  $\nabla_l$ for the incidence  matrix corresponding to $G_l$, and  $n_l  \,=\, \vert E_l \vert$.
	
	Throughout  for  a matrix  $X \in \mathbb{R}^{s 
		\times t}$.  If  $S \subset [s]$, $T \subset [t]$  we denote by $X_{S,T} $ the matrix in $\mathbb{R}^{ \vert S\vert \times  \vert T\vert }$   such that  $(X_{S,T})_{i,j} := X_{S(i),T(j)}$  for  $i \in [\vert S\vert ]$, $j \in [\vert T\vert ]$.

	Now, back to the proof, we recall the basic inequality,
	\begin{equation}
	\label{eqn:basic_inequality}
	\begin{array}{lll}
	\| \hat P - P_0 \|_F^2 &\le &2 \displaystyle \sum_{l^{\prime}=1}^{q}  \sum_{l=1}^{q}  \left| \langle R_{V_{l^{\prime}} , V_l }, (\hat P - P_0)_{V_{l^{\prime}} , V_l} \rangle_F \right| +   \\
	& &  \lambda \displaystyle\sum_{l^{\prime}=1}^{q} \sum_{ i \in V_ l^{\prime}}   \sum_{l=1}^{q} \left[ \| \nabla_{l} (P_0)_{i,\cdot} \|_1 +  \| \nabla_{l} (P_0)_{\cdot,i} \|_1  \right]     - \lambda \displaystyle \sum_{l^{\prime}=1}^{q} \sum_{ i \in V_{ l^{\prime} }}   \sum_{l=1}^{q}  \left[ \| \nabla_{l} \hat{P}_{i,\cdot} \|_1 +  \| \nabla_{l} \hat{P}_{\cdot,i} \|_1  \right]
	\end{array}
	\end{equation}
	
	Consider now  running depth first search (DFS) on $G_l$, and let $j_1^{l},\ldots,j_{n_l}^{l}$ be the ordering of $V_l$ such that $j_t^{l}$ is the $t$th vertex that the DFS visits (let the DFS start from an arbitrary node).
	Let $\nabla_{1D,l}$ denote the edge incidence matrix for the 1D chain graph that connects $j_t^l$ to $j_{t+1}^l$ for $t = 1,\ldots, n_l-1$. 
	Then by Lemma 1 in \cite{padilla2016dfs},  for any $l,l^{\prime} \in [q] $  we have that
	\begin{equation}
	\label{eq:DFSlemma}
	\sum_{i  \in V_{l^{\prime}} } \| \nabla_{1D,l} P_{i,\cdot} \|_1 \le 2 \sum_{i\in V_{l^{\prime}} } \| \nabla_l P_{i,\cdot} \|_1. 
	\end{equation}
	And we notice that
	\begin{equation}
	\label{eqn:dfs_trick}
	\sum_{i  \in V_{l^{\prime}} } \| \nabla_{1D,l} P_{i,\cdot} \|_1   +     	 \sum_{i  \in V_{l} } \| \nabla_{1D,l^{\prime}} P_{\cdot,i} \|_1 := \| \nabla_{2D,l^{\prime},l} P^{l^{\prime},l} \|_1   
	\end{equation}
	where  the latter is the total variation of  $P^{l^{\prime},l}  = \text{vec} (P_{V_{l^{\prime}} \times V_l  })$, along an appropriately constructed 2D grid graph of size $n_{l^{\prime}} \times n_l$. Here,  $\rm vec$   denotes the vectorization. 
	
	Now for  $l, l^{\prime} \in [q]$, 
	we define for  $x \in \mathbb{R}^{n \times n}$    the matrix   $S_{l^{\prime},l}(x) \in \mathbb{R}^{n \times n}$ as
	\[
	S_{l^{\prime},l}(x)_{i,j}  =  1\,\,\,\,\,\text{if } \,\, (i,j) \in V_{l^{\prime}} \times V_l,  \,\,\,\,\,\,S_{l^{\prime},l}(x)_{i,j}   =0\,\,\,\,\,\text{otherwise.}   
	\] 
	Let us consider the ordering $\tau$ obtained by concatenating the DFS orderings associated with  the $G_l$  graphs, see above.  Let   $\nabla_{2D}$  be the  incidence operator associated with  the  $n \times n$ 2D grid graph using  such ordering $\tau$. We also  write  $\Pi$ for the orthogonal projection on the span of $1_{n^{2}}$.
	Then by (\ref{eq:DFSlemma})  and (\ref{eqn:dfs_trick}),  Cauchy–Schwarz inequality, and H\"{o}lder inequality,
	\[
	\begin{array}{lll}
	\left| \langle R_{V_{l^{\prime}} , V_l }, (\hat P - P_0)_{V_{l^{\prime}} , V_l} \rangle_F \right| &=& \left| \langle  S_{l^{\prime},l}(R),  S_{l^{\prime},l}(\hat P - P_0)\rangle_F \right| \\
	& \le & \vert \langle \Pi \text{vec} \left(  S_{l^{\prime},l}(R)  \right), \text{vec} (S_{l^{\prime},l}(\hat P - P_0))   \rangle_F \vert  +  \\
	& &\vert \langle (\nabla_{2D}^+)^T \text{vec}\left(S_{l,l^{\prime}}(R) \right) ,  \nabla_{2D}\text{vec} (S_{l,l^{\prime}}(\hat P - P_0)) \rangle_F \vert \\ 
	&\le &  \| \Pi \text{vec} \left(  S_{l^{\prime},l}(R)  \right)\|_2\,\|(\hat P - P_0)_{V_{l^{\prime}},V_l}\|_2 \,+\,\\
	& & \| (\nabla_{2D}^+)^T \text{vec}\left(S_{l^{\prime},l}(R) \right)\|_{\infty} \|\nabla_{2D}\text{vec} (S_{l^{\prime},l}(\hat P - P_0))\|_1\\
	& \le&   \| \Pi \text{vec} \left(  S_{l^{\prime},l}(R)  \right)\|_2\,\|(\hat P - P_0)_{V_{l^{\prime}},V_l}\|_2 \,+\,\\
	& &2 \| (\nabla_{2D}^+)^T \text{vec}\left(S_{l^{\prime},l}(R) \right)\|_{\infty}\,\big[\displaystyle \sum_{i\in V_{l^{\prime}} } \| \nabla_l (\hat P - P_0)_{i,\cdot} \|_1 +\\
	& &\displaystyle \sum_{i\in V_{l} } \| \nabla_{l^{\prime}} (\hat P - P_0)_{\cdot,i} \|_1 +   2n\| \hat P - P_0\|_{\infty}   \Big].
	\end{array}
	\]
	On the other hand,  because the entries of $S_{l^{\prime},l}(R)$ are iid subGaussian (aside from those that were set to $0$),  for any $u>0$,
	\begin{equation*}
	\underset{l,l^{\prime} \in [q]}{\max }\,\,   \| \Pi \text{vec} \left(  S_{l^{\prime},l}(R)  \right)\|_2 \le   2\sigma\sqrt{2 \log(e q^2/u) }, \,\,\,\,\,\,\,\,  
	\end{equation*}
	and 
	\begin{equation*}
	\underset{l,l^{\prime} \in [q]}{\max }\,\,  	2 \| (\nabla_{2D}^+)^T \text{vec}\left(S_{l^{\prime},l}(R) \right)\|_{\infty}  \le  2\sigma\,\max_j \| (\nabla_{2D}^+)_{\cdot,j} \|_2\,\,\sqrt{2 \log\left( \frac{2\,e\, q^2\,n^2   }{u} \right) },
	\end{equation*}
	with probability  at least  $1-2u$.  
	
	Moreover,  by by Prop. 1 of \cite{hutter2016optimal} there exists a positive constant $C > 0$  such that
	\begin{equation*}
	\max_j \| (\nabla_{2D}^+)_{\cdot,j} \|_2 \le  C\sqrt{\log n}.
	\end{equation*}
	Therefore,  with probability at least $1-2u$, 
	\begin{equation*}
	\begin{array}{lll}
	\frac{1}{2} \| \hat P - P_0 \|_F^2 &\le & \displaystyle \sum_{l^{\prime}=1}^{q}  \sum_{l=1}^{q} \left[  4\sigma\sqrt{2 \log(e q^2/u) }\|(\hat P - P_0)_{V_{l^{\prime}},V_l}\|_2   - \frac{1}{2}\|(\hat P - P_0)_{V_{l^{\prime}},V_l}\|_2^2  \right]\\
	& & 4 \sigma C\,\sqrt{2 \log n \log\left(    \frac{2\,e\, q^2\,n^2   }{u} \right) } \displaystyle \sum_{l^{\prime}=1}^{q}  \sum_{l=1}^{q} \Big[ \displaystyle \sum_{i\in V_{l^{\prime}} } \| \nabla_l \hat P_{i,\cdot} \|_1 +   \sum_{i\in V_{l} } \| \nabla_{l^{\prime}} \hat P_{\cdot,i} \|_1 \\
	& &  \displaystyle \sum_{i\in V_{l^{\prime}} } \| \nabla_l  (P_0)_{i,\cdot} \|_1 +   \sum_{i\in V_{l} } \| \nabla_{l^{\prime}}  (P_0)_{\cdot,i} \|_1  +   2n\| \hat P - P_0\|_{\infty}  \Big] + 
	\\
	&& \lambda \displaystyle\sum_{l^{\prime}=1}^{q} \sum_{ i \in V_{l^{\prime}}}   \sum_{l=1}^{q} \left[ \| \nabla_{l} (P_0)_{i,\cdot} \|_1 +  \| \nabla_{l} (P_0)_{\cdot,i} \|_1  \right]     - \lambda \displaystyle \sum_{l^{\prime}=1}^{q} \sum_{ i \in V_{l^{\prime}}}   \sum_{l=1}^{q}  \left[ \| \nabla_{l} \hat{P}_{i,\cdot} \|_1 +  \| \nabla_{l} \hat{P}_{\cdot,i} \|_1  \right]  
	\end{array}
	\end{equation*}
	and so by the   inequality $2xa - a^2 \le x^2$, and  choosing $\lambda$ as
	\[
	\lambda  =  4 \sigma C\,\sqrt{2 \log n \log\left(    \frac{2\,e\, q^2\,n^2   }{u} \right) },
	\]
	we obtain
	\begin{equation*}
	\begin{array}{lll}
	\frac{1}{2} \| \hat P - P_0 \|_F^2 & \le &  8q^2 \sigma^2 \log(e q^2/u) +  \left[8\sigma C\,\sqrt{2 \log n \log\left(    \frac{2\,e\, q^2\,n^2   }{u} \right) } \right] nq^2\| \hat P - P_0\|_{\infty}   + \\
	& & 8\sigma C\,\sqrt{2 \log n \log\left(    \frac{2\,e\, q^2\,n^2   }{u} \right) } \displaystyle\sum_{l^{\prime}=1}^{q} \sum_{ i \in V_{l^{\prime}}}   \sum_{l=1}^{q} \left[ \| \nabla_{l} (P_0)_{i,\cdot} \|_1 +  \| \nabla_{l} (P_0)_{\cdot,i} \|_1  \right]  
	\end{array}
	\end{equation*}
	with probability at least $1-2u$.

\end{proof}

\begin{proof}[Proof of Corollary \ref{cor:graphon}.]
	First,   we observe  that  $\hat{P}_{i,j} \in [0,1]$  for all $i,j \in [n]$. If not, then  both the  loss and the objective in the definition of $\hat P$  can be improve  by setting 
	$\tilde{P} \in \mathbb{R}^{n \times n}$  as
	\[
	\tilde{P}_{i,j} =   \begin{cases}
	1  & \text{if }\,\,\,\,\,\hat P_{i,j} > 1\\
	0  & \text{if }\,\,\,\,\,\hat  P_{i,j} <0\\
	\hat   P_{i,j}   &  \text{otherwise.} 
	\end{cases}
	\]
	Hence,  $\|  \hat P - P\|_{\infty}   \leq  1 $.
	
	Next,  let  $S  \,=\,\{a_1,\ldots,a_{m-1}  \}$ and $I(\delta)  \,=\,  \{ i\,:\,     \vert \xi_i -  x\vert > \delta,\,\,\,\,    \forall  x \in \mathcal{S}  \}$  for some  $\delta >0$.
	
	Let $p_{\min} > \delta > 0$ and $B_\delta(u) = [u - \delta,u+\delta]$ and 
	\begin{equation*}
	\tilde B(i,\delta) = \{ j \ne i : \xi_j \in B_\delta(\xi_i) \}.
	\end{equation*}
	Notice that the $\tilde B(i,\delta)$ have the same distribution for  all $i$.
	By Proposition 27 from \cite{von2010hitting},
	\begin{equation*}
	\end{equation*}
	\begin{align*}
	&\mathbb P\left\{ |\tilde B(1,\delta)| \le \frac{\delta n}{4} \right\} = \int_0^1 \mathbb P \left( \left\vert \left\{ j > 2 : \xi_j \in B_\delta(x)  \right\} \right\vert \le \frac{\delta n}{4} \right) dx \\
	&\quad \le \int_0^1 \mathbb P \left( \left\vert \left\{ j > 2 : \xi_j \in B_\delta(x) \cap A_{l(x)} \right\} \right\vert \le \frac{n-1}{2} {\rm Vol}(B_\delta(x) ) \right) dx \\
	&\quad \le \int_0^1 \exp \left( - \frac{n {\rm Vol}(B_\delta(x) )}{24} \right) dx\\
	&\quad = \exp \left( - \frac{n\delta}{12} \right).
	\end{align*}
	Let us consider the event 
	\begin{equation*}
	\Omega(\delta) = \bigcap_{i=1}^n \left\{ |\tilde B(i,\delta)| \ge \frac{\delta n}{4} \right\}.
	\end{equation*}
	Set $\delta = 4 K / n + \Delta_n$.  By the union bound $\mathbb P \Omega(\delta) \rightarrow 1$ if $\delta n /\log n \rightarrow \infty$ (which is satisfied under our assumptions for $K$) so henceforth assume $\Omega(\delta)$.
	Let $\xi_{i'} \in (\xi_i - \delta, \xi_i + \delta)$,  with $i \in I(\delta)$ then by Assumption \ref{as:hat_upper}, 
	\begin{equation*}
	\label{eq:dupperNN}
	\hat d(i,i') \le L_2 |\xi_i - \xi_{i'}| + \Delta_n \le \frac{4 L_2 K}{n}  +  (L_2 +1)\Delta_n,
	\end{equation*}
	and notice that on $\Omega(\delta)$ there are at least $K$ such vertices $i'$.  
	
	On the other hand, let  $C>0$   such that 
	\[
	\delta'  :=  C\left(   \frac{K}{n}   + \Delta_n \right)   >  \frac{1}{L_1}\left( \frac{4 L_2 K}{n}  + (L_2+2)\Delta_n   \right).
	\]
	Hence, by Assumption \ref{as:hat_lower}, if  $i \in I(\delta')$, then  $\vert \xi_{i'}  - \xi_i \vert > \delta'  $  implies  $i'$  is not among the KNN of $i$. 
	
	Without loss of generality, suppose that $\xi_1 \le \xi_2 \le \ldots \le \xi_n$ (we can always reorder $\xi$), and let $N_i(K)$ be the KNN of $i$.
	Let $\tilde N_i  =\{  j   \neq i \,:\, \vert \xi_{i'}  - \xi_i \vert \leq  \delta'     \}  $, and notice that $\max_i |\tilde N_i| = O_P((K + n \Delta_n))$, this follows again  by Proposition 27 from \cite{von2010hitting}. Similarly, $\vert [n] \backslash  I(\delta')\vert    =  O_P( K + n \Delta_n )$.
	
	Therefore,
	\begin{align*}
	\| \nabla P_0 \|_1 &= \sum_{j,i} \sum_{i' \in N_i(K)} |f(\xi_i,\xi_j) - f(\xi_{i'},\xi_j)| \\ 
	& \le \sum_{j \in [n],i \in  I(\delta') } \sum_{i' \in N_i(K)} \sum_{k = i \wedge i'}^{i \vee i' - 1} |f(\xi_k,\xi_j) - f(\xi_{k+1},\xi_j)|    \,\, +\,\,  n\,K\, \vert [n]  \backslash   I(\delta')\vert  \\ 
	& \le \sum_{j,k} K \cdot (\max_i |\tilde N_i|) \cdot |f(\xi_k,\xi_j) - f(\xi_{k+1},\xi_j)|  +  n\,K\, \vert [n]  \backslash   I(\delta')\vert\\
	& = O_P \left( n (K^2 + n K \Delta_n)\right),
	\end{align*}
	by Assumption \ref{as:bv}.
	The above display follows from the fact that the term $|f(\xi_k,\xi_j) - f(\xi_{k+1},\xi_j)|$ appears at most $K \cdot (\max_i |\tilde N_i|)$ times. 
	The same holds for $\| \nabla P_0^\top \|_1$.
\end{proof}

\begin{proof}[Proof of Corollary \ref{cor:graphon_dhat}]
Consider the random variable $Z = (A_{i,l} - A_{j,l}) A_{k,l}$ then we have that $|Z| \le 1$ and so by Hoeffding's inequality,
\begin{equation*}
\mathbb P\{ |(A_i - A_j)^\top A_k - (P_i - P_j)^\top P_k| \ge \sqrt n u | \xi \} \le 2 \exp (-2 u^2),
\end{equation*}
where conditional on $\xi$ means that we fix the latent parameters and draw the matrix $A$.

By eq. 19 in \cite{zhang2015estimating}, with probability at least $1 - 2 n^{-\gamma / 4}$,
\begin{equation*}
\max_{i \ne j} | (A^2/n)_{i,j} - (P^2/n)_{i,j} | \le \left( \frac{(C + \gamma) \log n}{n} \right)^{\frac 12},
\end{equation*}
for some  positive constant $C$.

Next define 
\begin{equation*}
\tilde d^2(i,j) : = \frac{1}{n(n-2)} \sum_{k\ne i,j} |(P_i - P_j)^\top P_k |.   
\end{equation*}
Also, $(P_i - P_j)^\top P_k = (P^2)_{i,k} - (P^2)_{j,k}$.
Hence, 
\begin{multline*}
\max_{i,j} \left| \hat d^2_1(i,j) - \tilde d^2_1(i,j)\right| \le \max_{i,j,k:k \ne i,j} | (A^2/n - P^2/n)_{i,k} - (A^2/n - P^2/n)_{j,k} | \\
\le 2 \max_{i,k: k \ne i} | (A^2/n)_{i,k} - (P^2/n)_{i,k} |.
\end{multline*}
Hence, we have that 
\begin{equation*}
\max_{i,j} \left| \hat d_1^2(i,j) - \tilde d_1^2(i,j)\right| = O_P \left( \sqrt{\frac{\log n}{n}} \right). 
\end{equation*}
Define 
\begin{equation*}
{d'}_1^2(i,j) := \frac{1}{n-2} \sum_{k \ne i,j} \left| \int (f(\xi_i , u) - f(\xi_j, u)) f(\xi_k ,u) du \right| 
\end{equation*}
Similarly to the above considerations, $n |\tilde d^2_1(i,j) - {d'}^2_1(i,j)|$ is bounded by
\begin{multline*}
\frac{1}{n-2} \sum_{k \ne i,j} \left| \sum_l \left( (f(\xi_i,\xi_l) - f(\xi_i,\xi_l)) f(\xi_k,\xi_l) - \int (f(\xi_i,u) - f(\xi_j,u)) f(\xi_k,u) du \right) \right| \\ 
\le 2 \max_{i \ne k} \left| \sum_l \left( f(\xi_i,\xi_l) f(\xi_k,\xi_l) - \int f(\xi_i,u) f(\xi_k,u) du \right) \right| \\ 
\le 2 \max_{i \ne k} \left|f(\xi_i,\xi_k) (f(\xi_i,\xi_i) + f(\xi_k,\xi_k)) - 2\int f(\xi_i,u) f(\xi_k,u) du \right| \\
+ 2 \max_{i \ne k} \left| \sum_{l \ne i,k} \left( f(\xi_i,\xi_l) f(\xi_k,\xi_l) - \int f(\xi_i,u) f(\xi_k,u) du \right) \right|.
\end{multline*}
The first term is bounded by a constant and we can control the second term by Hoeffding's inequality,
\begin{equation*}
\mathbb P \left\{ \left| \sum_{l \ne i,k} \left( f(\xi_i,\xi_l) f(\xi_k,\xi_l) - \int f(\xi_i,u) f(\xi_k,u) du \right) \right| \ge u \sqrt{n-2} \right\} \le 2 \exp ( - 2 u^2).
\end{equation*}
Hence,
\begin{equation*}
\max_{i,j} |\tilde d^2_1(i,j) - {d'}^2_1(i,j)| = O_P\left( \sqrt{\frac{\log n}{n}} \right).
\end{equation*}
Furthermore, by Hoeffding's inequality,
\begin{equation*}
\mathbb P \left\{ | {d'}^2_1(i,j) - \mathbb E[ {d'}^2_1(i,j) |\xi_i,\xi_j] | \ge u/\sqrt{n-2} \right \} \le 2 e^{-2 u^2}.
\end{equation*}
By definition, $ \mathbb E[ {d'}^2_1(i,j) |\xi_i,\xi_j] = d_1^2(i,j)$, hence,
\begin{equation*}
\max_{i,j} | {d'}_1^2(i,j) - d_1^2(i,j) | = O_P\left(\sqrt{\frac{\log n}{n}} \right).
\end{equation*}
These combined give us that 
\begin{equation*}
\max_{i,j} |\hat d^2_1(i,j) - d^2_1(i,j)| \le \Delta_n
\end{equation*}
for some sequence $\Delta_n \asymp \sqrt{(\log n)/n}$.

Hence, if $d_1$ satisfies Assumptions \ref{as:hat_lower}, \ref{as:hat_upper} with $\Delta_n \asymp \sqrt{(\log n) / n}$ then $\hat d_1$ does as well with high probability. 
\end{proof}

\subsection{Projected Newton}

For Distributed ADMM Algorithm, the update of $P$ and $Q$ is implemented with projected Newton on each column separately. The general projected Newton method is not guaranteed to converge. But for some special cases, the projected Newton method converges. In our case, the dual form of \eqref{eq:FL} has the box constraint problem and is one of this kind of problems.
Dual problem for the fused lasso problem is 
\begin{equation}
\label{equ:dual}
\min_{u}\frac{1}{2}\|\nabla^Tu\|_2^2 - u^T\nabla y, \ \ \ \ s.t. ||u||_\infty \leq \frac{\lambda}{2}
\end{equation}
where $y$ is the column of $P$ or $Q$ in Section 2.1. Then the primal solution is $y - \nabla^Tu$.  For reproducibility purposes we include the projected Newton algorithm that we employ.

\begin{algorithm}
\caption{Projected Newton}
\begin{algorithmic}[1]
  \State Solve $\nabla\nabla^Tu = \nabla y$ by Algorithm Laplacian solver for a graph.
  \If{$||u||_\infty \leq \frac{\lambda}{2}$},{ return $u$.} \EndIf
  \State $u = \textnormal{proj}(u)$.
  \State Let $\Delta = \frac{\lambda}{2}||\nabla y||_1 - u^T\nabla y$, the duality gap,
  \While{$\Delta$ is large}
  \State Active set $I = \{i:u_i = - \frac{\lambda}{2}\ \ \  \textnormal{and}\ \ \  (\nabla\nabla^Tu - \nabla y)_i > 0 \ \ \ \textnormal{or} \ \ \  u_i = \frac{\lambda}{2} \ \ \ \textnormal{and}\ \ \  (\nabla\nabla^Tu - \nabla y)_i < 0\}$,
  \State $\bar{I} = [n]/I$,
  \State Solve $(\nabla_{\bar{I}}\nabla_{\bar{I}}^T)s = (\nabla_{\bar{I}}\nabla^Tu - \nabla_{\bar{I}}y)$,
  \State Update $u_{\bar{I}} = \textnormal{proj}(u_{\bar{I}} - \alpha s_{\bar{I}})$, $\alpha$ is chosen with backtracking line search.
  \EndWhile
  \State return $u$.
\end{algorithmic}
\end{algorithm}

\begin{algorithm}
\caption{Laplacian solver for a graph}
\begin{algorithmic}[1]
  \State For each connected component \textit{C} of graph $G$ with incidence matrix $\nabla$, solve the Laplacian linear system restricted on \textit{C}. That is, solve $\nabla_\textit{C}^T\nabla_\textit{C} z_\textit{C} = y - P_{\textnormal{null}(\nabla_\textit{C})}(y)$, where $\nabla_\textit{C}$ is the incidence matrix restricted on \textit{C}, and $P_{\textnormal{null}(\nabla_\textit{C})}(y)$ is the projection of $y$ on the null space of $\nabla_\textit{C}$,
  \State return $\nabla z = (\nabla (z_\textit{C}) \textrm{ for all connected comp. }C \textrm{ of } G)$
 \end{algorithmic}
\end{algorithm}

\end{document}